\documentclass{article}
\usepackage{graphicx,geometry,amsmath,amssymb,caption,subcaption,enumerate,url,amsthm}
\usepackage{tikz}
\usetikzlibrary{decorations.pathreplacing}
\usetikzlibrary{positioning}

\title{Uncertainty propagation through trained multi-layer perceptrons: Exact analytical results}
\author{Andrew Thompson and Miles McCrory\footnote{National Physical Laboratory, Hampton Road, Teddington, TW11 0LW, UK;\\ \indent\;\; andrew.thompson@npl.co.uk; miles.mccrory@npl.co.uk.}}
\date{January 2026}

\newcommand{\RR}{\mathbb{R}}
\newcommand{\EE}{\mathbb{E}}

\newcommand{\ra}{\rightarrow}

\begin{document}

\maketitle

\abstract{We give analytical results for propagation of uncertainty through trained multi-layer perceptrons (MLPs) with a single hidden layer and ReLU activation functions. More precisely, we give expressions for the mean and variance of the output when the input is multivariate Gaussian. In contrast to previous results, we obtain exact expressions without resort to a series expansion.}

\section{Introduction}\label{sec:intro}

The importance of assessing the trustworthiness of machine learning (ML) models through uncertainty quantification is now widely recognised -- for example the EU AI Act indirectly mandates uncertainty quantification by requiring transparency and risk management~\cite{deuschel2024role}. Here we focus on the challenge of \emph{uncertainty propagation through trained machine learning regression models}, in which known uncertainties concerning the input data are propagated through a fixed (trained) model to obtain uncertainties for the output prediction. Despite not taking into account contributions to the prediction uncertainty due to the training process, this approach is of interest for two reasons. Firstly, in some testing/validation scenarios the tester may only have access to the ML model as a `black box', in which case uncertainty propagation through a fixed model is the only feasible approach that can be taken. Secondly, uncertainty propagation through a fixed model can be used to carry out sensitivity analysis of the model~\cite{borgonovo2016sensitivity}.

There are two main types of method for propagating uncertainties through models: analytical and sampling-based. In the analytical approach, a mathematical expression is derived either for the distribution of the output or some summary statistic such as variance. In the sampling-based approach, the model is evaluated on samples of the input data distribution, for example Monte Carlo sampling~\cite{bipm2008supplement}, thereby obtaining an empirical distribution.

If it is possible to derive analytical expressions characterising the output uncertainty, there are several reasons why it is usually preferable to use them. Firstly, analytical expressions are often more accurate, since they often precisely characterise the output uncertainty, in contrast to sampling-based approaches which only approximate the true distribution of the model output with accuracy increasing with the number of samples. Secondly, analytical expressions provide more transparency by providing mathematical insight and understanding into the sources of propagated uncertainty. Thirdly, analytical approaches are more easily reproducible, since all that is needed to reproduce the evaluation is the input data distribution and the analytical expression for its propagation, whereas sampling-based methods often require a record of the sampling process. 

In this paper, we focus on analytical characterisation of output uncertainty for fully-connected feed-forward regression neural networks, also known as multi-layer perceptrons (MLPs), with a single hidden layer and with the popular choice of a ReLU activation function. The usefulness of single-hidden layer networks is underlined by theoretical results showing that they are universal approximators if sufficiently wide~\cite{hornik1989multilayer}, and that they converge to Gaussian Processes in the infinite-width limit if the weights are initialised in a particular way~\cite{neal1996priors}. More precisely, we derive analytical expressions for the mean and variance of model predictions under the assumption that the input data follows a multivariate Gaussian distribution. 

We demonstrate the practicality of our methods by illustrating their use upon a test problem involving the prediction of the state-of-health of lithium-ion cells based upon Electrical Impedance Spectroscopy (EIS) data. By comparing our methods with a Monte Carlo sampling approach~\cite{bipm2008supplement} in this context, we validate the correctness of our expressions.

\subsection{Paper structure}

The structure of the paper is as follows. We first review closely related work in Section~\ref{related}. In Section~\ref{architecture}, we define the single-hidden-layer MLP model that we are considering, and then we present our analytical results for uncertainty propagation in Section~\ref{analytical_results}. Proofs of all novel results can be found in the Appendix. We present our numerical experiments in Section~\ref{experiments}, before giving our conclusions in Section~\ref{discussion}.

\section{Related work}\label{related}

The closest related existing work is~\cite{wright2024analytic}, in which analytical expressions are given for the pairwise covariances after passing multivariate Gaussian inputs through general activation functions. The results apply to any elementwise independent and identical activation function whose mean function when applied to univariate Gaussian input is infinitely differentiable with respect to the mean of its input. This includes the ReLU activation function, and also others such as Heaviside and GELU. The expressions for the covariances take the form of infinite Taylor series, which the authors point out can be computed to arbitrary precision. However, the number of terms required to compute to a desired precision is dependent upon the covariances of the inputs, and each of the terms in the series have to be painstakingly calculated for a specific activation function. Our results, on the other hand, do not involve infinite series but are closed-form functions of standard (1D and 2D) Gaussian integrals. Our results therefore provide a simpler method for computing the output moments to arbitrary precision for the case of a ReLU activation function.

To the authors' best knowledge, all other analytical methods that have been proposed for propagating uncertainty through neural networks are not exact but involve approximation at some point. For neural networks with more than one hidden layer, the simplifying assumption is often made that the inputs to each activation step are multivariate Gaussian, which is an approximation except for the first hidden layer. Even for propagation through a single hidden layer, approximations are often employed, for example assuming the inputs are mean zero~\cite{bibi2018analytic}, using asymptotic approximations~\cite{wu2018deterministic} and linearising the activation function~\cite{diamzon2025uncertainty}. 

Finally, we note that various sampling-based methods have been proposed to approximately propagate uncertainty through neural networks, including Monte Carlo methods~\cite{monchot2023input}, Kalman filtering~\cite{titensky2018uncertainty} and the Unscented Transform~\cite{monchot2023input}.

\section{Neural network architecture}\label{architecture}

The neural network architecture that we consider in this paper is shown in Figure~\ref{fig:perceptron}. We consider an MLP with a single hidden layer and a ReLU activation function. A vector of inputs $V\in\RR^m$ is passed through an affine convolution to give $W = A^T V + c \in \mathbb{R}^p$, where $p$ is the number of hidden units. $W$ is then passed through a componentwise ReLU to give $X=W_+\in\RR^p$, where the ReLU function is defined in the usual way as $\mathrm{ReLU}(x):=\max(x,0)$. Finally, $X$ is passed through an affine convolution to obtain the output $Y = \beta^T X + d \in \mathbb{R}$.

\begin{figure}[h!]
	\centering
    \begin{tikzpicture}[shorten >=1pt]
        \tikzstyle{unit}=[draw,shape=circle,minimum size=1.15cm]
 
        \node[unit](v1) at (0,3.5){$v_1$};
        \node[unit](v2) at (0,2){$v_2$};
        \node(dots) at (0,1.125){\vdots};
        \node[unit](vm) at (0,0){$v_m$};

        \node[unit](w1) at (4,3.5){$w_1$};
        \node[unit](w2) at (4,2){$w_2$};
        \node(dots) at (4,1.125){\vdots};
        \node[unit](wp) at (4,0){$w_p$};
 
        \node[unit](x1) at (8,3.5){$x_1$};
        \node[unit](x2) at (8,2){$x_2$};
        \node(dots) at (8,1.125){\vdots};
        \node[unit](xp) at (8,0){$x_p$};

        \node[unit](y1) at (12,2){$y_1$};

        \draw[->] (v1) -- (w1);
        \draw[->] (v1) -- (w2);
        \draw[->] (v1) -- (wp);
 
        \draw[->] (v2) -- (w1);
        \draw[->] (v2) -- (w2);
        \draw[->] (v2) -- (wp);
 
        \draw[->] (vm) -- (w1);
        \draw[->] (vm) -- (w2);
        \draw[->] (vm) -- (wp);

        \draw[->] (w1) -- (x1);

        \draw[->] (w2) -- (x2);
 
        \draw[->] (wp) -- (xp);

        \draw[->] (x1) -- (y1);
        \draw[->] (x2) -- (y1);
        \draw[->] (xp) -- (y1);

        \draw [decorate,decoration={brace,amplitude=10pt},xshift=-4pt,yshift=0pt](-0.5,4) -- (0.75,4) node [black,midway,yshift=+0.6cm]{$V \in \mathbb{R}^m$};
        \draw [decorate,decoration={brace,amplitude=10pt},xshift=-4pt,yshift=0pt](3.5,4) -- (4.75,4) node [black,midway,yshift=+0.6cm]{$W = A^TV + c \in \mathbb{R}^p$};
        \draw [decorate,decoration={brace,amplitude=10pt},xshift=-4pt,yshift=0pt](7.5,4) -- (8.75,4) node [black,midway,yshift=+0.6cm]{$X = W_{+} \in \mathbb{R}^p$};
        \draw [decorate,decoration={brace,amplitude=10pt},xshift=-4pt,yshift=0pt](11.5,2.5) -- (12.75,2.5) node [black,midway,yshift=+0.6cm]{$Y = \beta^T X + d \in \mathbb{R}$};
    \end{tikzpicture}
    \caption{Architecture of an MLP with a single hidden layer and ReLU activation function.}
    \label{fig:perceptron}
\end{figure}
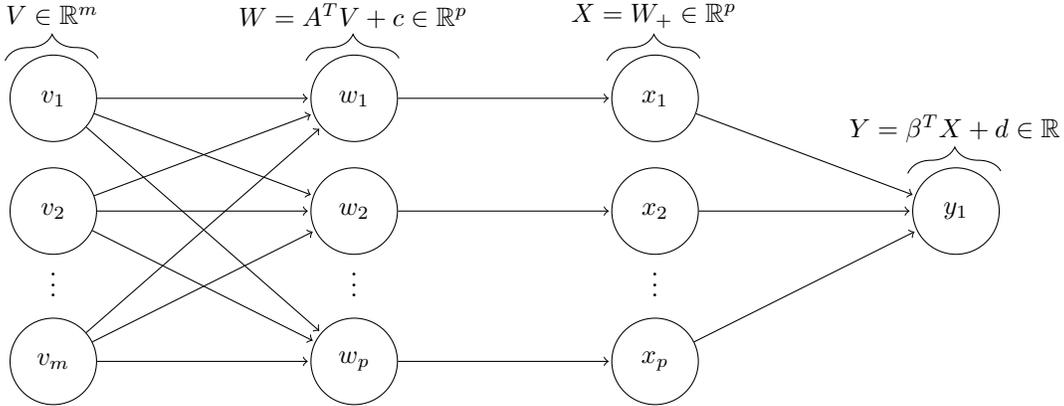

\section{Analytical uncertainty propagation results}\label{analytical_results}

In this section we present our analytical expressions for the mean and variance of the output of the neural network defined in Section~\ref{architecture} when the input data follows a multivariate Gaussian distribution. Suppose then that the input data $V\in\RR^m$ is multivariate Gaussian with mean $\lambda$ and covariance $\Lambda$.

We will see that the main challenge is propagating uncertainty through the ReLU activation function as propagating uncertainty through the convolutional layers is relatively straightforward. For the first convolution, it is a standard result that $W$ is also multivariate Gaussian with mean $\mu:=A^T\lambda+c$ and covariance $\Sigma:=A^T\Lambda A$~\cite[Result 13.2]{johnson2002applied}. For the second convolution, suppose that the mean and covariance of $X=W_+$ are known; let us write $\gamma$ and $\Gamma$ for the mean and covariance of $X$ respectively. Then the mean of $Y$ is $\beta^T \gamma + d$ and the variance of $Y$ is $\beta^T \Gamma \beta$~\cite[Equation 2-43]{johnson2002applied}. 

It remains, therefore, to obtain expressions for $\gamma$ and $\Gamma$, which are the mean and covariance of a \emph{rectified} multivariate Gaussian distribution~\cite{kan2017moments}. We have
$$\begin{array}{rcl}\gamma_i &=& \mathbb{E} X_i,\\ \Gamma_{ii} &=& \mathbb{E} X_i^2 - (\mathbb{E} X_i)^2 = \mathbb{E} X_i^2 - \gamma_i^2,\\
\Gamma_{ij} &=& \mathbb{E} X_i X_j - \mathbb{E} X_i \mathbb{E} X_j = \mathbb{E} X_i X_j - \gamma_i \gamma_j\;\textrm{for}\;i\neq j,\end{array}$$ and so it suffices to obtain expressions for $\mathbb{E} X_i$, $\mathbb{E} X_i^2$ and $\mathbb{E} X_i X_j$ where $\begin{bmatrix}X_i&X_j\end{bmatrix}^T$ follows a rectified bivariate Gaussian distribution.

Define $\phi(x)$ to be the standard normal probability density function (pdf), define $\Phi(x)$ to be the standard normal cumulative distribution function (cdf), and write $\sigma^2_i=\Sigma_{ii}$ for $i=1,\ldots,p$. We have the following results for $\mathbb{E} X_i$ and $\mathbb{E} X_i^2$ respectively.

\newtheorem{theorem}{Theorem}
\begin{theorem}[\textbf{\cite[Equation (C.9)]{frey1999variational}}]
    Let $X$ be defined as above and suppose $\sigma_i>0$. Then, for all $i=1,\ldots,p$, 
    \begin{equation}\label{EXi_eqn}\mathbb{E} X_i = \sigma_i \phi \left( \frac{\mu_i}{\sigma_i} \right)+\mu_i \Phi \left( \frac{\mu_i}{\sigma_i} \right)\end{equation}
    and
    \begin{equation}\label{EXi2_eqn}\mathbb{E}X_i^2 =\mu_i \sigma_i \phi \left( \frac{\mu_i}{\sigma_i} \right)+(\mu_i^2+\sigma_i^2) \Phi \left( \frac{\mu_i}{\sigma_i} \right).\end{equation}
\label{thrm:xi}
\end{theorem}

Before stating our result for $\mathbb{E} X_i X_j$, we introduce some further notation. When $\sigma_i,\sigma_j>0$, write $$\rho_{ij} = \frac{\Sigma_{ij}}{\sigma_i \sigma_j}$$ 
for the correlation between $W_i$ and $W_j$ for all $i,j \in \{1,\hdots,p\}$, and define, for $|\rho|<1$, $\phi_2 (x; \rho)$ and $\Phi_2 (x; \rho)$ to be the pdf and cdf respectively of a bivariate Gaussian distribution with zero mean, unit variance and correlation $\rho$. If in addition $|\rho_{ij}|<1$, write 
\begin{equation}\label{omega_def}
\omega_{ij} = \frac{\mu_i\sigma_j - \rho_{ij} \mu_j\sigma_i}{\sigma_i\sigma_j\sqrt{1 - \rho_{ij}^2}}.
\end{equation}

\begin{theorem}
    Let $X$ be defined as above and suppose $\sigma_i,\sigma_j>0$ and $|\rho_{ij}|<1$. Then, for $i \neq j \in \{1,\ldots,p\}$,
    \begin{align}
    \mathbb{E} X_i X_j &= \mu_j \sigma_i \phi \left( \frac{\mu_i}{\sigma_i} \right) \Phi ( \omega_{ji} ) + \mu_i \sigma_j \phi \left( \frac{\mu_j}{\sigma_j} \right) \Phi ( \omega_{ij} )\nonumber\\
    &+ \sigma_i\sigma_j(1-\rho_{ij}^2) \phi_2 \left( \frac{\mu_i}{\sigma_i}, \frac{\mu_j}{\sigma_j}; \rho_{ij} \right)+ ( \mu_i \mu_j + \rho_{ij} \sigma_i \sigma_j ) \Phi_2 \left( \frac{\mu_i}{\sigma_i}, \frac{\mu_j}{\sigma_j}; \rho_{ij} \right).\label{EXiXj}
    \end{align}     
 
\label{thrm:xixj}
\end{theorem}

A proof of this result is given in Appendix~\ref{proofs}.

We note that the expressions in Theorems~\ref{thrm:xi} and~\ref{thrm:xixj} are closed-form functions of standard (1D and 2D) Gaussian integrals. This is in contrast to the expressions in~\cite{wright2024analytic} which take the form of infinite Taylor series.

It is straightforward to show that, in the special case of $\rho_{ij}=0$, (\ref{EXiXj}) reduces to
$$ \mathbb{E} X_i X_j=\left[\sigma_i \phi \left( \frac{\mu_i}{\sigma_i} \right)+\mu_i \Phi \left( \frac{\mu_i}{\sigma_i} \right)\right]\left[\sigma_j \phi \left( \frac{\mu_j}{\sigma_j} \right)+\mu_j \Phi \left( \frac{\mu_j}{\sigma_j} \right)\right].$$
This is nothing other than $\mathbb{E} X_i\,\mathbb{E} X_j$ by Theorem~\ref{thrm:xi}, which is expected in the case of zero correlation (independence).

The assumptions in Theorems \ref{thrm:xi} and \ref{thrm:xixj} ignore certain corner cases in which there is either zero uncertainty in the input variables or perfect correlation between the variables in the hidden layer. For completeness, we show in Appendix~\ref{rho1_proof} how the results in this section extend to these corner cases.

\section{Validation of the analytical expressions}\label{experiments}

In this section, we present evidence that the analytical approach to uncertainty propagation through a single-hidden layer MLP that was presented in Section~\ref{analytical_results} gives results which are consistent with a Monte Carlo sampling approach. 

We do this in the context of a test problem that has been used in the related publications~\cite{thompson2024analytical} and~\cite{thompson2025analytical}. This test problem concerns the estimation of the state-of-health (SOH) of lithium-ion cells by combining Electrical Impedance Spectroscopy (EIS) measurements with so-called \emph{equivalent circuit models}. Full background on this test problem can be found in~\cite{chan2022comparison}. 

A public dataset consisting of SOH and EIS measurements is available at [14]. As in~\cite{thompson2024analytical} and~\cite{thompson2025analytical}, the dataset was reduced to only include samples for which the SOH is above 75 \%, which resulted in 165 samples. An equivalent circuit model was used in~\cite{chan2022comparison} to reduce the raw EIS data to 12 coefficients. Following this approach, it was found in~\cite{thompson2024analytical} that SOH may be estimated using just two of these coefficients as input into a Gaussian Process. Here, we instead train an MLP with a single hidden layer on these two input variables. In~\cite{chan2022comparison}, uncertainties in the EIS raw data due to temperature fluctuations were propagated through the equivalent circuit model using Monte Carlo sampling. Based on this approach, uncertainties on the two input variables were modelled in~\cite{thompson2024analytical} using a multivariate Gaussian distribution. Full details on these data preprocessing steps may be found in~\cite{thompson2024analytical}.

Single hidden layer MLP models with 12 hidden units were trained using MATLAB's \texttt{fitrnet} function with default settings. For the MLP model, a random 75:25 train/test split is used, giving a training set size of 124 and a test set size of 41. The training set input data is centred and normalised, and the same transformation is then applied to the test set input data. Since it is only realistic to expect an ML model to interpolate the training data, one sample outside the range of SOH values in the training set was removed from the test set, leaving 40 samples.

For each sample in the test set, the mean and variance of the output of the trained MLP model is computed both by means of the analytical expressions from Section~\ref{analytical_results} and using Monte Carlo sampling with number of samples $n$ varied over $\{10^3,10^4,10^5,10^6\}$. We measure agreement between the analytical and Monte Carlo computations using RMSE, and we write $\kappa(\EE\,Y)$ and $\kappa(\textrm{Var}\,Y)$ for the RMSEs for the means and the variances respectively.

Table~\ref{accuracy_table} gives the RMSEs over the 40 test data points for the means and variances for the trained MLP model. We observe that the RMSE decreases with the number of Monte Carlo trials, and that it is below $10^{-5}$ for both quantities with $10^6$ trials. 

If the estimates from Monte Carlo sampling are converging to the true value, we would expect the error to be proportional to $1/\sqrt{n}$~\cite{caflisch1998monte}. In Figure~\ref{loglog} we display log-log plots of the RMSE against number of Monte Carlo trials $n$ for the estimates of both the mean and the variance. We observe a straight line indicating a power law relationship, and the gradient of the lines of best fit for $\kappa(\EE\,Y)$ and $\kappa(\textrm{Var}\,Y)$ are respectively -0.5036 and -0.4966, indicating proportionality to $1/\sqrt{n}$ in both cases.

\begin{figure}[h!]
     \centering
     \begin{subfigure}[b]{0.45\textwidth}
         \centering
         \includegraphics[width=\textwidth]{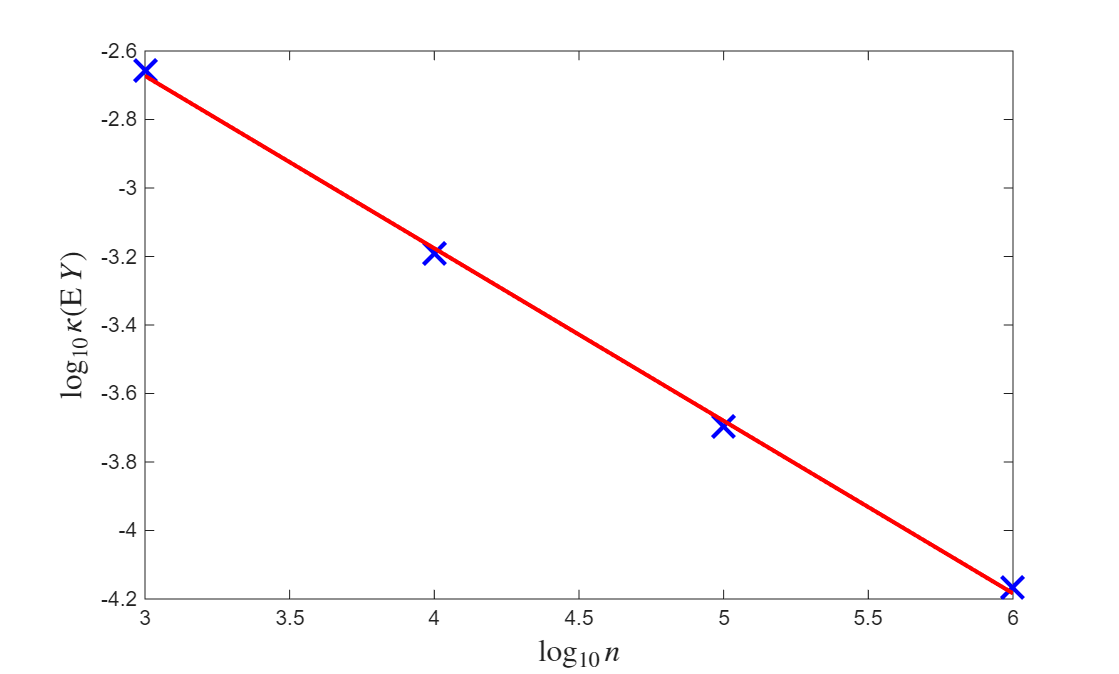}
         \caption{$\EE\,Y$.}
         \label{loglog_mu}
     \end{subfigure}
     \begin{subfigure}[b]{0.45\textwidth}
         \centering
         \includegraphics[width=\textwidth]{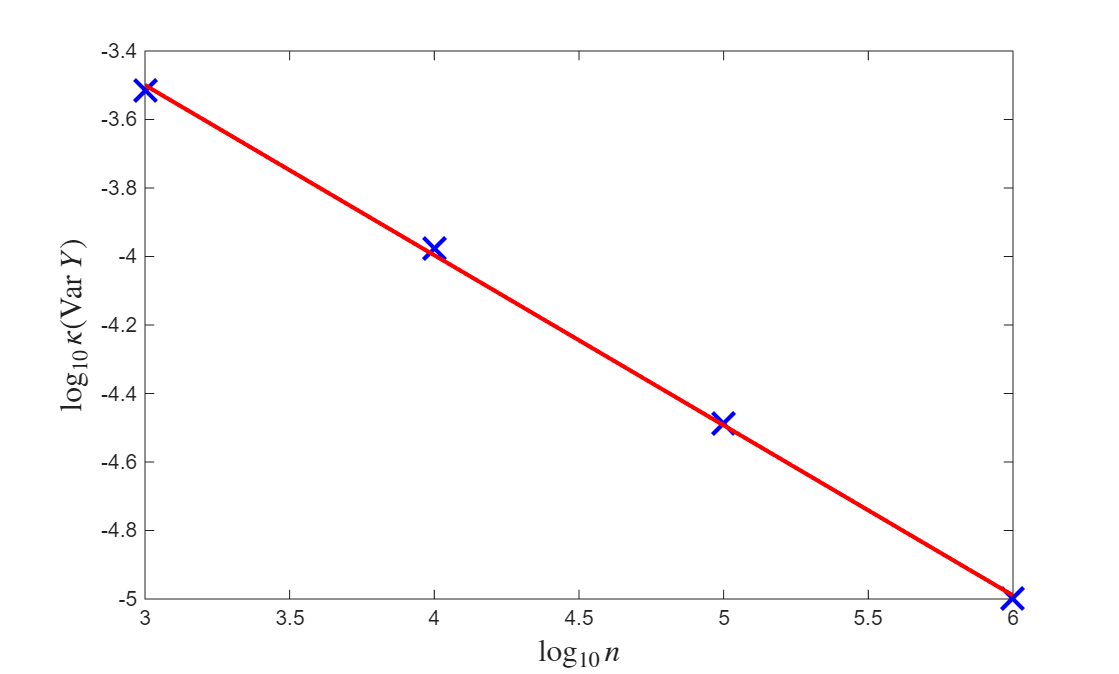}
         \caption{$\textrm{Var}\,Y$.}
         \label{loglog_var}
     \end{subfigure}
     \caption{Log-log plots of RMSE against number of Monte Carlo samples.}\label{loglog}
\end{figure}

These results provide strong evidence in a realistic setting that the analytical expressions are in agreement with Monte Carlo sampling.

\begin{table}[h!]
    \centering
    \begin{tabular}{|c|cc|}
    \hline
        Number of MC trials & $\kappa(\EE\,Y)\times 10^2$ &  $\kappa(\textrm{Var}\,Y)\times 10^3$\\
        \hline
        $10^3$ & 0.2200 & 0.3058\\
        $10^4$ & 0.0645 & 0.1054\\
        $10^5$ & 0.0201 & 0.0326\\
        $10^6$ & 0.0068 & 0.0100\\
    \hline
    \end{tabular}
    \caption{RMSE over the test set for varying numbers of Monte Carlo trials for multivariate Gaussian input variables.}
    \label{accuracy_table}
\end{table}

\section{Conclusions}\label{discussion}

We have presented analytical expressions for the mean and variance of the output of a trained MLP with ReLU activation functions in the case of multivariate Gaussian input. We have also validated the expressions through a comparison with Monte Carlo sampling. The use of analytical expressions in this context offers a solution which is more accurate, transparent and reproducible than a Monte Carlo sampling approach.

It would be possible to extend these results to other activation functions related to the ReLU, for example the Heaviside step function and the \emph{leaky ReLU}, which can be decomposed into the summation of a ReLU and a linear function. Extension to other activation functions that are less closely related to the ReLU, or to other input distributions, is likely to require different mathematical techniques.

Another interesting question would be to compare the computational complexity of our analytical approach for single-hidden-layer MLPs with sampling-based approaches, but this is left as future work.

Extension of the analytical results presented here to MLPs with more hidden layers would involve fully characterising the multivariate distribution after the second convolution in such a way that the moments after the second activation function could be derived, and this is a more challenging task. Alternatively, our results could be used in a multi-layer uncertainty propagation by making Gaussian approximations in each layer as in the approach in~\cite{wright2024analytic}, but this approach does not give an exact analytical characterisation of the output moments.  

\section*{Acknowledgment}

This work was supported by the UK Government’s Department for Science, Innovation and Technology (DSIT). The authors would like to thank Philip Aston (NPL) for reviewing a draft of the paper and providing useful feedback.

\appendix

\section{Proof of Theorem~\ref{thrm:xixj}}\label{proofs}

We first establish the following lemma.\\

\newtheorem{lemma}{Lemma}
\begin{lemma}
Let $X$ be defined as in Section~\ref{analytical_results}. Then 
\begin{equation}\label{truncated_prob}\mathbb{P}(X_i > 0 \cap X_j > 0) = \Phi_2 \left( \frac{\mu_i}{\sigma_i}, \frac{\mu_j}{\sigma_j}; \rho_{ij} \right).
\end{equation}
\label{lem:conprob}
\end{lemma}

\begin{proof}
   
We have $\mathbb{P}(X_i > 0 \cap X_j > 0) = \mathbb{P}(W_i > 0 \cap W_j > 0)$ where $W \sim \mathcal{N}(\mu, \Sigma)$. Defining $V_i = \frac{W_i-\mu_i}{\sigma_i}, V_j = \frac{W_j-\mu_j}{\sigma_j}$, we have
\begin{equation}\label{eqn:standard_normal}\left[ {V_i \atop V_j} \right] \sim \mathcal{N} \left( \left[ {0 \atop 0} \right], \left[ {1 \atop \rho_{ij}} {\rho_{ij} \atop 1} \right] \right).\end{equation}
We then have
\begin{eqnarray}\mathbb{P}(W_i > 0 \cap W_j > 0) &=& \mathbb{P}(\sigma_i V_i + \mu_i > 0 \cap \sigma_j V_j + \mu_j > 0)\nonumber\\&=& \mathbb{P} \left( V_i > - \frac{\mu_i}{\sigma_i} \cap V_j > - \frac{\mu_j}{\sigma_j} \right)\nonumber\\&=&\mathbb{P} \left( V_i < \frac{\mu_i}{\sigma_i} \cap V_j < \frac{\mu_j}{\sigma_j} \right),\label{eqn:lemma1}\end{eqnarray}
where the last line follows by the symmetry of $V_i$ and $V_j$ about $0$. The result now follows by combining (\ref{eqn:standard_normal}) and (\ref{eqn:lemma1}).
    
\end{proof}

\noindent \textbf{Proof of Theorem~\ref{thrm:xixj}:} In Section 2 of the Online Appendix of \cite{kan2017moments}, expressions for the moments of the truncated multivariate Gaussian distribution are given. In particular, an expression for $\mathbb{E} (X_i X_j | X_i > a_i \cap X_j > a_j )$ is given in the special case of $\sigma_i^2 = \sigma_j^2 = 1$, and it is explained how to transform the expression for the case of general $\sigma_i^2$ and $\sigma_j^2$. Setting $a_i = a_j = 0$,
\begin{equation}
\begin{split}
\mathbb{E}(X_i X_j | X_i > 0 \cap X_j > 0) & = \mu_i \mu_j + \rho_{ij}\\
& + \frac{\mu_j \phi(\eta_i) \Phi (\tilde{\omega}_{ji}) + \mu_i \phi (\eta_j) \Phi (\tilde{\omega}_{ij})}{\Phi_2 (\eta_i, \eta_j ; \rho_{ij})}\\
& + \frac{(1- \rho_{ij}^2) \phi_2 (\eta_i, \eta_j ; \rho_{ij})}{\Phi_2 ( \eta_i, \eta_j ; \rho_{ij})},
\end{split}
\label{kenrobotti}
\end{equation}
where $\eta_i = \mu_i - a_i$, $\eta_j = \mu_j - a_j$ and where 
$$\tilde{\omega}_{ij} = \frac{\mu_i - \rho_{ij} \mu_j}{\sqrt{1 - \rho_{ij}^2}}.$$
It is stated in~\cite{kan2017moments} that (\ref{kenrobotti}) may be extended to general $\sigma_i^2$ and $\sigma_j^2$ by substituting $\mu_i$ and $\mu_j$ with $\frac{\mu_i}{\sigma_i}$ and $\frac{\mu_j}{\sigma_j}$ respectively throughout and then multiplying the whole expression by $\sigma_i\sigma_j$. These substitutions mean that $\tilde{\omega}_{ij}$ and $\tilde{\omega}_{ji}$ should be replaced by $\omega_{ij}$ and $\omega_{ij}$ respectively, which were defined in (\ref{omega_def}). Applying these changes gives     
\begin{eqnarray}
\mathbb{E}(X_i X_j | X_i > 0 \cap X_j > 0) & =& \mu_i \mu_j + \rho_{ij} \sigma_i \sigma_j + \frac{\mu_j \sigma_i \phi \left( \frac{\mu_i}{\sigma_i} \right) \Phi (\omega_{ji}) + \mu_i \sigma_j \phi \left( \frac{\mu_j}{\sigma_j} \right) \Phi (\omega_{ij})}{\Phi_2 \left( \frac{\mu_i}{\sigma_i}, \frac{\mu_j}{\sigma_j} ; \rho_{ij} \right)}\nonumber\\
& + &\frac{\sigma_i \sigma_j(1-\rho_{ij}^2)\phi_2 \left( \frac{\mu_i}{\sigma_i}, \frac{\mu_j}{\sigma_j} ; \rho_{ij} \right)}{\Phi_2 \left( \frac{\mu_i}{\sigma_i}, \frac{\mu_j}{\sigma_j} ; \rho_{ij} \right)}.\label{truncated}
\end{eqnarray}
By the law of total expectation $\mathbb{E}X_i X_j = \mathbb{E}(X_i X_j | X_i > 0 \cap X_j > 0) \mathbb{P}(X_i > 0 \cap X_j > 0)$. Thus, we obtain the final result by multiplying (\ref{truncated}) by the expression in (\ref{truncated_prob}).
\hfill$\Box$

\section{Extension of analytical results to corner cases}\label{rho1_proof}

\subsection{The case of $\sigma_i=0$ or $\sigma_j=0$}

Theorem~\ref{thrm:xi} assumes that $\sigma_i>0$. In the case $\sigma_i=0$, we simply pass a deterministic $\mu_i$ through the ReLU function, giving 
$$\EE X_i=\left\{\begin{array}{ll}\mu_i&\mu_i>0\\
0&\mu_i\le 0,\end{array}\right.$$
and
$$\EE X_i^2=\left\{\begin{array}{ll}\mu_i^2&\mu_i>0\\
0&\mu_i\le 0.\end{array}\right.$$
We note that these results can also be obtained by taking the limit of the expressions in Theorem~\ref{thrm:xi} as $\sigma\ra 0$.

Similarly, when $\sigma_i=\sigma_j=0$, we have
$$\EE X_{ij}=\left\{\begin{array}{ll}\mu_i\mu_j&\mu_i>0\;\textrm{and}\;\mu_j>0\\
0&\textrm{otherwise}.\end{array}\right.$$

In the case $\sigma_i=0$ and $\sigma_j>0$, we have
$$\EE X_i X_j=\left\{\begin{array}{ll}\mu_i\cdot\EE X_j=\mu_i\left[\sigma_j \phi \left( \frac{\mu_j}{\sigma_j} \right)+\mu_j \Phi \left( \frac{\mu_j}{\sigma_j} \right)\right]&\mu_i>0\\0&\textrm{otherwise},\end{array}\right. $$
and by symmetry an analogous result holds for the case $\sigma_i>0$ and $\sigma_j=0$.

\subsection{The case of $|\rho_{ij}|=1$}

Theorem~\ref{thrm:xixj} assumes in addition that $|\rho_{ij}|<1$. We also derive a result for the limiting case of $|\rho_{ij}|=1$, for which we will require some preliminary lemmas. We first require a characterisation of the bivariate Gaussian distribution in the $|\rho_{ij}|=1$ limit.

\begin{lemma}[\textbf{\cite[Chapter IV, Theorem 16]{mood1950introduction}}]
\label{rho_pm1_lem}
Let $W$ be defined as in Section~\ref{analytical_results}. Then, in the $|\rho_{ij}|=1$ limit, $W_i$ and $W_j$ are linearly dependent in an affine sense (degenerate distribution), such that  
\begin{equation}\label{W_result}
W_j=\frac{\rho_{ij}\sigma_j}{\sigma_i}(W_i-\mu_i)+\mu_j.
\end{equation}
\end{lemma}

We next establish the following lemma.

\begin{lemma}\label{epsilon_lem}
Suppose $W_i\sim\mathrm{N}(\mu_i,\sigma_i^2)$, $Z\sim\mathrm{N}(0,1)$ independently of $W_i$. Given $\epsilon\geq 0$, define \begin{equation}\label{Wj_def}
W_j(\epsilon):=\frac{\rho_{ij}\sigma_j}{\sigma_i}(W_i-\mu_i)+\mu_j+\epsilon Z,
\end{equation}
where $|\rho_{ij}|=1$. Then, for $\epsilon>0$, define
\begin{equation}\label{muj_sigmaj_def}
\hat{\mu}_j(\epsilon):=\mathbb{E}\,W_j,\;\;\hat{\sigma}_j^2(\epsilon):=\mathrm{Var}\,W_j,
\end{equation}
\begin{equation}\label{Sigmaij_rhoij_def}
\hat{\Sigma}_{ij}(\epsilon):=\mathrm{Cov}\,(W_i,W_j),\;\;\hat{\rho}_{ij}(\epsilon):=\frac{\hat{\Sigma}_{ij}(\epsilon)}{\sigma_i\cdot\hat{\sigma}_j(\epsilon)},
\end{equation}
and
\begin{equation}\label{omegaij_def}
\hat{\omega}_{ij}(\epsilon):=\frac{\mu_i\hat{\sigma}_j(\epsilon) - \hat{\rho}_{ij}(\epsilon) \hat{\mu}_j(\epsilon)\sigma_i}{\sigma_i\cdot\hat{\sigma}_j(\epsilon)\sqrt{1 - \hat{\rho}_{ij}(\epsilon)^2}}.
\end{equation}
Then, for $\epsilon>0$, we have
\begin{equation}\label{lem_result}
\hat{\omega}_{ij}(\epsilon)=\frac{1}{\sigma_i\epsilon}(\mu_i\sigma_j-\rho_{ij}\mu_j\sigma_i)+\mathcal{O}(\epsilon).
\end{equation}
\end{lemma}

\noindent \textit{Proof:} With the definitions in (\ref{Wj_def}) and (\ref{muj_sigmaj_def}), we have 
\begin{equation}\label{Wj_moments}
\mathbb{E}\, W_j(\epsilon)=\mu_j,\;\;\mathrm{Var}\,W_j(\epsilon)=\sigma_j^2+\epsilon^2,
\end{equation}
and additionally with the definition in (\ref{Sigmaij_rhoij_def}) we then have\begin{eqnarray*}\mathbb{E}\left[W_i W_j(\epsilon)\right]&=&\mathbb{E}\left\{W_i\left[\displaystyle\frac{\rho_{ij}\sigma_j}{\sigma_i}(W_i-\mu_i)+\mu_j+\epsilon Z\right]\right\}\\
&=&\displaystyle\frac{\rho_{ij}\sigma_j}{\sigma_i}\left(\mathbb{E}\,W_i^2-\mu_i\mathbb{E}\,W_i\right)+\mu_j\mathbb{E}\,W_i\\
&=&\displaystyle\frac{\rho_{ij}\sigma_j}{\sigma_i}(\sigma_i^2+\mu_i^2-\mu_i^2)+\mu_i\mu_j\\
&=&\rho_{ij}\sigma_i\sigma_j+\mu_i\mu_j,\end{eqnarray*}
where in the second line we use the independence of $W_i$ and $Z$, and from which it follows that
\begin{equation}\label{rhoij}
\hat{\rho}_{ij}(\epsilon)=\frac{\rho_{ij}\sigma_j}{\sqrt{\sigma_j^2+\epsilon^2}}.
\end{equation}
Substituting (\ref{Wj_moments}) and (\ref{rhoij}) into (\ref{omegaij_def}) then gives, after a little rearrangement,
\begin{equation}\label{wij_result}\hat{\omega}_{ij}(\epsilon)=\frac{1}{\sigma_i\epsilon}\left(\mu_i\sqrt{\sigma_j^2+\epsilon^2}-\displaystyle\frac{\rho_{ij}\mu_j\sigma_i\sigma_j}{\sqrt{\sigma_j^2+\epsilon^2}}\right).
\end{equation}
Expanding the expression inside braces in (\ref{wij_result}) as a Taylor series in $\epsilon$ then yields (\ref{lem_result}).\hfill$\Box$\\
$ $

We will also need a lemma characterising the standard bivariate Gaussian cdf function $\Phi(x;\rho)$ in the limit as $\rho_{ij}\ra\pm1$.

\begin{lemma}
Let $\Phi(x)$ and $\Phi_2(x;\rho)$ be as defined in Section~\ref{analytical_results}. Then
\begin{equation}\label{Phi2_rho1}
\Phi_2\left( \frac{\mu_i}{\sigma_i}, \frac{\mu_j}{\sigma_j}; \rho_{ij} \right)\ra\Phi\left[\min\left(\frac{\mu_i}{\sigma_i},\frac{\mu_j}{\sigma_j}\right)\right]\;\textrm{as}\;\rho_{ij}\ra 1,
\end{equation}
and
\begin{equation}\label{Phi2_rho-1}
\Phi_2\left( \frac{\mu_i}{\sigma_i}, \frac{\mu_j}{\sigma_j}; \rho_{ij} \right)\ra\max\left[\Phi\left(\frac{\mu_i}{\sigma_i}\right)+\Phi\left(\frac{\mu_j}{\sigma_j}\right)-1,0\right]\;\textrm{as}\;\rho_{ij}\ra -1.\end{equation}
\end{lemma}

\noindent \textit{Proof:} Substituting $\mu_i=\mu_j=0$ and $\sigma_i=\sigma_j=1$ into (\ref{W_result}), we have $W_j=W_i$ when $\rho_{ij}=1$ and $W_i=-W_j$ when $\rho_{ij}=-1$. We thus have
$$\lim_{\rho\ra 1}\Phi_2\left( \frac{\mu_i}{\sigma_i}, \frac{\mu_j}{\sigma_j}; \rho_{ij} \right)=\mathbb{P}\left(W_i\le \frac{\mu_i}{\sigma_i},W_i\le\frac{\mu_j}{\sigma_j}\right)=\Phi\left[\min\left(\frac{\mu_i}{\sigma_i},\frac{\mu_j}{\sigma_j}\right)\right]$$
and
$$\lim_{\rho\ra -1}\Phi_2\left( \frac{\mu_i}{\sigma_i}, \frac{\mu_j}{\sigma_j}; \rho_{ij} \right)=\mathbb{P}\left(-\frac{\mu_j}{\sigma_j}\le W_i\le \frac{\mu_i}{\sigma_i}\right)=\max\left[\Phi\left(\frac{\mu_i}{\sigma_i}\right)+\Phi\left(\frac{\mu_j}{\sigma_j}\right)-1,0\right],$$
which proves the result.\hfill $\Box$\\
$ $

Before proceeding to our result for $\EE X_i X_j$ when $|\rho_{ij}|=1$, we fix our definition of the Heaviside function $H(x)$~\cite{hoskins2009delta} to be
$$H(x):=\left\{\begin{array}{ll}1&x>0\\
\frac{1}{2}&x=0\\0&x<0.\end{array}\right.$$

\begin{theorem}\label{rho1}
Let $X$ be defined as above. If $\rho_{ij}=1$, then, for $i \neq j \in \{1,\ldots,p\}$,
\begin{align}\EE X_i X_j&=H(\mu_i\sigma_j-\mu_j\sigma_i)\mu_i \sigma_j \phi \left( \frac{\mu_j}{\sigma_j} \right)+[1-H(\mu_i\sigma_j-\mu_j\sigma_i)]\mu_j \sigma_i \phi \left( \frac{\mu_i}{\sigma_i}\right)\nonumber\\&+( \mu_i \mu_j + \sigma_i \sigma_j )\Phi\left[\min\left(\frac{\mu_i}{\sigma_i},\frac{\mu_j}{\sigma_j}\right)\right].\label{rho_plus1}\end{align}
Meanwhile if $\rho_{ij}=-1$, then, for $i \neq j \in \{1,\ldots,p\}$,
\begin{align}\EE X_i X_j&=H(\mu_i\sigma_j+\mu_j\sigma_i)\left[\mu_i \sigma_j \phi \left( \frac{\mu_j}{\sigma_j}\right) +\mu_j \sigma_i \phi \left( \frac{\mu_i}{\sigma_i}\right)\right]\nonumber\\&+( \mu_i \mu_j - \sigma_i \sigma_j )\max\left[\Phi\left(\frac{\mu_i}{\sigma_i}\right)+\Phi\left(\frac{\mu_j}{\sigma_j}\right)-1,0\right].\label{rho_minus1}\end{align}
\end{theorem}

\noindent \textit{Proof:} By Lemma~\ref{rho_pm1_lem}, (\ref{W_result}) holds in the $|\rho_{ij}|=1$ limit. We can determine the limiting behaviour of $\omega_{ij}$ by considering instead $W_j(\epsilon)$ as defined in (\ref{Wj_def}). It then follows by Lemma~\ref{epsilon_lem} that 
$$\hat{\omega}_{ij}(\epsilon)=\frac{1}{\sigma_i\epsilon}(\mu_i\sigma_j-\rho_{ij}\mu_j\sigma_i)+\mathcal{O}(\epsilon),$$
from which it follows that
\begin{equation}\label{omega_lim}\lim_{\epsilon\ra 0}\Phi[\hat{\omega}_{ij}(\epsilon)]=\left\{\begin{array}{ll}1&\mu_i\sigma_j-\rho_{ij}\mu_j\sigma_i>0\\\frac{1}{2}&\mu_i\sigma_j-\rho_{ij}\mu_j\sigma_i=0\\0&\mu_i\sigma_j-\rho_{ij}\mu_j\sigma_i<0.\end{array}\right.\end{equation}
Substituting $\rho_{ij}=\pm 1$, (\ref{omega_lim}), (\ref{Phi2_rho1}) and (\ref{Phi2_rho-1}) into (\ref{EXiXj}) now gives the required result.\hfill$\Box$\\
$ $

We remark that (\ref{rho_plus1}) reduces to (\ref{EXi2_eqn}) in the case $\mu_i=\mu_j$ and $\sigma_i=\sigma_j$. This is expected, since in this case $X_i=X_j$.

\bibliography{refs}
\bibliographystyle{plain}

\end{document}